\documentclass{article}

\usepackage[preprint]{neurips_2024}

\usepackage[utf8]{inputenc} 
\usepackage[T1]{fontenc}    
\usepackage{hyperref}       
\usepackage{url}            
\usepackage{booktabs}       
\usepackage{amsfonts}       
\usepackage{nicefrac}       
\usepackage{microtype}      
\usepackage{xcolor}         
\usepackage{amsmath}
\usepackage{amsthm}
\usepackage{multirow}
\usepackage{graphicx}
\usepackage{float}
\usepackage{color}
\newtheorem{theorem}{Theorem}[section]
\newtheorem{definition}{Definition}[section]
\newtheorem{lemma}{Lemma}[section]

\title{Solving Oversmoothing in GNNs via Nonlocal Message Passing: Algebraic Smoothing and Depth Scalability}

\author{%
   Weiqi Guan\thanks{School of Mathematical Sciences, Fudan University, Shanghai 200433, China} \\
  \texttt{wqguan24@m.fudan.edu.cn} \\
  \And
  Junlin He\thanks{Corresponding author}\\
  \texttt{junlinspeed.he@connect.polyu.hk}
}

\begin{document}

\maketitle

\begin{abstract}

    The relationship between Layer Normalization (LN) placement and the oversmoothing phenomenon remains underexplored. We identify a critical dilemma: \textbf{Pre-LN} architectures avoid oversmoothing but suffer from the \textit{curse of depth}, while \textbf{Post-LN} architectures bypass the curse of depth but experience oversmoothing.

    To resolve this, we propose a new method called non-local message passing based on Post-LN that induces algebraic smoothing, preventing oversmoothing without the curse of depth. Empirical results across five benchmarks demonstrate that our approach supports deeper networks (up to 256 layers) and improves performance, requiring no additional parameters.

    Key contributions:
    \begin{itemize}
        \item \textbf{Theoretical Characterization:} Analysis of LN dynamics and their impact on oversmoothing and the curse of depth.
        \item \textbf{A Principled Solution:} A parameter-efficient method that induces algebraic smoothing and avoids oversmoothing and the curse of depth.
        \item \textbf{Empirical Validation:} Extensive experiments showing the effectiveness of the method in deeper GNNs.
    \end{itemize}


\end{abstract}

\section{Introduction}
Graph Neural Networks (GNNs) have emerged as the de facto standard for learning from graph-structured data, ranging from motif detection\citep{besta2022motif} to social network analysis\citep{fan2019graph}. 
Despite their widespread success, GNNs face a fundamental limitation regarding their depth: \textit{over-smoothing}. 
This phenomenon manifests as the exponential convergence of node representations to a distinguishing-less subspace as the number of layers increases, leading to severe performance degradation in deep architectures.
While foundational models such as Graph Convolutional Networks (GCN) \citep{kipf2017semisupervised} and Graph Attention Networks (GAT) \citep{veličković2018graph} remain dominant baselines \citep{luo2024classicgnnsstrongbaselines}, they are theoretically susceptible to such issue \citep{NEURIPS2023_6e4cdfdd,Oono2020Graph, cai2020note}.

Significant research efforts have been dedicated to mitigating over-smoothing\citep{chen2020simple, Zhao2020PairNorm:, scholkemper2025residualconnectionsnormalizationprovably},particularly through the lens of differential equation analysis. 
Notably, GNNs can be interpreted as discretizations of continuous differential equations, with foundational works establishing connections to heat equations and diffusion processes\citep{chamberlain2021grand, pmlr-v97-wu19e}. 
While refining the discretization step size allows for deeper architectures, it fails to fundamentally prevent the exponential decay of Dirichlet energy associated with diffusive dynamics. 
Consequently, recent studies have sought to identify dynamical systems that inherently avoid such decay. 
Pioneering work by \citet{pmlr-v162-rusch22a} introduced wave equations on graphs, which exhibit oscillatory behavior rather than converging exponentially to a steady state. 
Subsequently, \citet{rusch2023gradient} incorporated a gradient gating mechanism to achieve algebraically smoothing, while \citet{kang2024unleashing} recently leveraged fractional time derivatives to attain similar effects. 
However, a critical limitation persists across these methods: the reliance on layer-shared edge weights (time-invariant parameters). 
This architectural constraint—prevalent in both explicit schemes\citep{rusch2023gradient} and implicit schemes where the step size functions as the layer index\citep{pmlr-v162-rusch22a, kang2024unleashing}—substantially limits the model's representational capacity and flexibility.


Inspired by the transformative success of Large Language Models (LLMs)~\citep{vaswani2017attention}, which can be conceptually viewed as models operating on fully-connected graphs, researchers have increasingly adapted the Transformer architecture to graph data to capture long-range dependencies\citep{rampášek2023recipegeneralpowerfulscalable, shirzad2023exphormersparsetransformersgraphs, wu2023nodeformerscalablegraphstructure}. 
Since Transformers can be formally viewed as fully connected GNNs employing self-attention\citep{joshi2025transformersgraphneuralnetworks}, it is natural to investigate whether they inherit the over-smoothing pathologies of traditional GNNs. 
A critical component governing the signal propagation and trainability of Transformers is Layer Normalization (LN). 
Extensive research in the NLP domain has revealed that the placement of LN—specifically Pre-LN versus Post-LN configurations—dramatically impacts model convergence and stability\citep{wang2019learning, xu2019understanding}. 
For instance, Pre-LN is often associated with better training stability but suffers from the \textit{curse of depth}, a phenomenon first formalized in\citep{sun2025cursedepthlargelanguage}. This term describes the observation where deeper layers in modern LLMs contribute significantly less to representation learning compared to earlier ones. In contrast, while Post-LN architectures can lead to rank collapse or convergence to a single cluster\citep{shi2022revisiting, geshkovski2025mathematicalperspectivetransformers}, they appear immune to the curse of depth, although this property has primarily been empirically observed in shallow configurations\citep{sun2025cursedepthlargelanguage}.

Despite these insights in LLMs, the direct causal relationship between Layer Normalization placement and the over-smoothing phenomenon specifically within attention-based GNNs remains largely unexplored. Surprisingly, we find a critical dilemma: \textbf{Pre-LN} configurations succumb to the \textit{curse of depth} despite avoiding signal collapse (over-smoothing), whereas \textbf{Post-LN} architectures avoid the curse of depth but suffer from rapid \textit{over-smoothing}. Figure \ref{fig: log-log plots of Laplacian energy}(c) reveals a power-law scaling of Laplacian energy in \textbf{Pre-LN} architectures, a behavior driven by the dominance of residual connections in the signal propagation. Although this mechanism averts the collapse of Laplacian energy (over-smoothing), it induces a side effect called \textit{curse of depth}: representations undergo minimal evolution across layers, as corroborated by the high inter-layer cosine similarity in Figure \ref{fig: inter-layer cosine similarity}(b). In sharp contrast, \textbf{Post-LN} architectures place greater emphasis on the residual branch, which serves to diversify features across layers. Nevertheless, they suffer from an exponential decay of high-frequency components as shown in Figure \ref{fig: log-log plots of Laplacian energy}(a), a phenomenon that leads to the over-smoothing issue. These observations align with those reported in LLMs \citep{sun2025cursedepthlargelanguage, shi2022revisiting, geshkovski2025mathematicalperspectivetransformers}.

\begin{figure}
  \centering
  \includegraphics[width=1\textwidth]{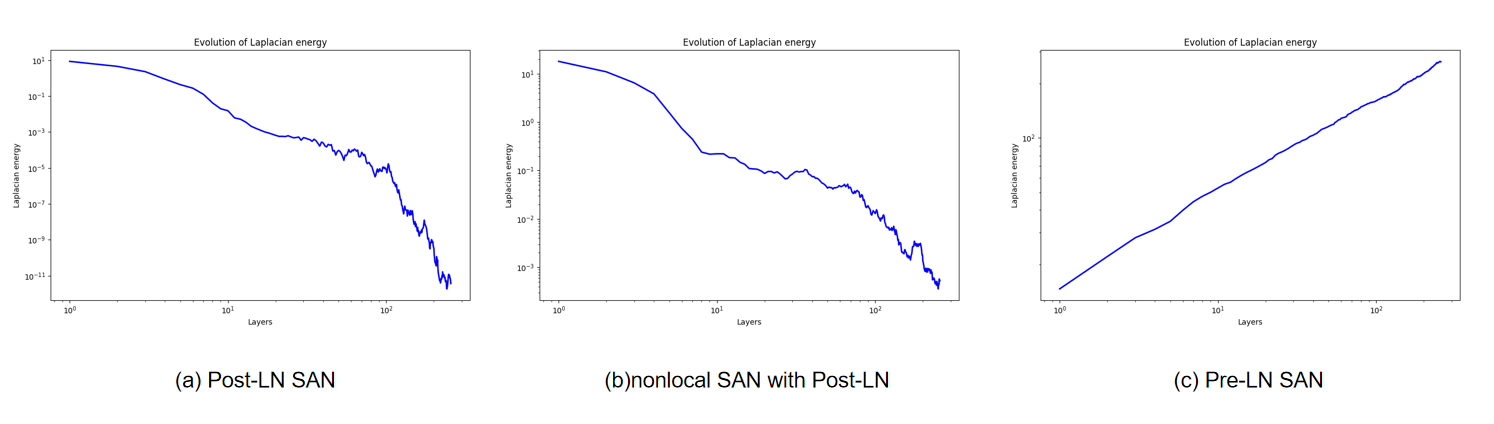}
  \label{fig: log-log plots of Laplacian energy}
  \caption{Log-log plots of Laplacian energy evolution at initialization on Cora. The panels display (a) the Post-LN SAN, (b) our proposed nonlocal SAN with Post-LN, and (c) the Pre-LN SAN architecture.}
\end{figure}

\begin{figure}
  \centering
  \includegraphics[width=0.8\textwidth]{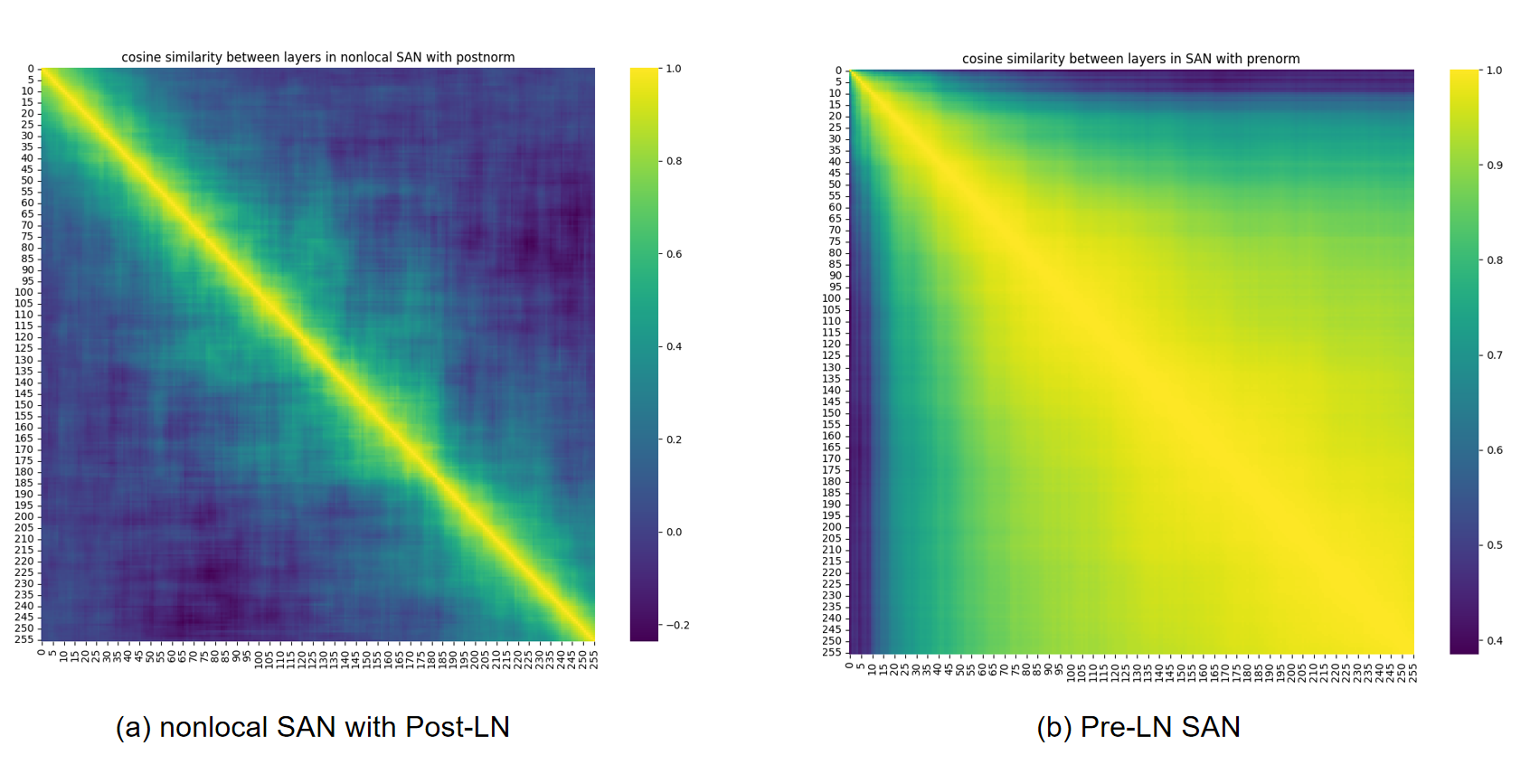}
  \caption{Inter-layer cosine similarity of representations after training on Cora. (a) Our proposed nonlocal SAN with Post-LN. (b) The Pre-LN SAN.}
  \label{fig: inter-layer cosine similarity}
\end{figure}

To rigorously analyze the layer-wise behavior of GNNs, we establish a novel theoretical framework by modeling GNNs as generalized weighted graphs and employing integration by parts tools. Within this framework, we first demonstrate that the weighted aggregation operation inherent in standard GNNs inevitably leads to over-smoothing, characterized by an exponential decay of the Laplacian energy as the network depth increases.
Furthermore, we extend our analysis to the Pre-LN configuration. While the property that the Laplacian energy exhibits a power-law growth pattern renders Pre-LN architectures robust against over-smoothing, we identify a critical limitation: the relative Laplacian energy decays to zero. This estimate of decay precipitates the curse of depth, hindering the effective training of extremely deep networks.
To resolve this dilemma, we turn our attention to the Post-LN architecture, which naturally avoids the curse of depth but traditionally suffers from over-smoothing. Building upon Post-LN, we propose a principled method with formal guarantees that enforces a power-law decay of the Laplacian energy. This approach achieves algebraic smoothing, thereby fundamentally mitigating the over-smoothing issue without introducing the optimization difficulties associated with the curse of depth.
As illustrated in Figure \ref{fig: log-log plots of Laplacian energy}(b), the Laplacian energy in our model decays slowly, maintaining a significant magnitude of $10^{-5}$ even at a depth of 256 layers. Crucially, our proposed method effectively circumvents the curse of depth, as corroborated by the low inter-layer cosine similarity shown in Figure \ref{fig: inter-layer cosine similarity}(a). Extensive experiments validate both our theoretical analysis and the superiority of the proposed method.
Our main contributions are summarized as follows:
\begin{itemize}
    \item \textbf{Theoretical Characterization of LN Dynamics:} 
    We perform a systematic investigation into the impact of Layer Normalization (LN) placement on the dichotomy between \textit{over-smoothing} and the \textit{curse of depth}. Within a theoretical framework based on weighted graphs, we provide a formal analysis using integration by parts to delineate the asymptotic behaviors of Laplacian energy.

    \item \textbf{A Principled Remedy for Over-smoothing:} 
    We propose a novel, parameter-efficient, and theoretically guaranteed method called nonlocal message passing to fundamentally overcome over-smoothing in Post-LN architectures. By inducing algebraic smoothing rather than exponential decay, our method requires no additional parameters and incurs negligible computational overhead. Importantly, our model remains immune to the curse of depth.

    \item \textbf{Empirical Verification and Scalability:}
    We conduct comprehensive experiments across 5 benchmarks to corroborate our theoretical findings. The empirical results demonstrate that our method effectively mitigates the over-smoothing issue, enabling the successful training of significantly deeper networks (up to at least 256 layers), avoiding the curse of depth, improving performance.
\end{itemize}

\section{Preliminaries}
\label{gen_inst}
In this section, we review the fundamental analysis on weighted undirected graphs $G=(V,E,\omega,\mu)$. The graph structure consists of a vertex set $V$, an edge set $E$, a symmetric weight function $\omega$ on $E$ satisfying $\omega_{ij}=\omega_{ji}>0$ for every $(i,j),(j,i)\in E$, and a vertex measure $\mu$ on $V$ with $\mu(i)>0$ for all $i\in V$, we also abbreviate $\mu(i)$ as $\mu_i$. Denote $n$ as the number of nodes. If node $j$ is a neighborhood of node $i$, we denote as $i\sim j$. The neighborhood of node $i$ is denoted as $\mathcal{N}_i:=\{j\in[n]:j\sim i\}$. The space of vector-valued functions on $G$ is denoted by $C(G)=\{X:V\to\mathbb{R}^d\}$. Readers may be more familiar with the expression for features on graphs as matrix $X\in \mathbb{R}^{n\times d}$, and can view $X(i)$ as the i-th row vectors and $u(i,j)$ as $(i,j)$ elements in matrix $X$. For two functionals $\mathcal{E}_1,\mathcal{E}_2: C(G)\rightarrow \mathbb{R}$ we denote $\mathcal{E}_1\sim \mathcal{E}_2$ if there exist universal constants $C_1,C_2$ such that $C_1 \mathcal{E}_1(X)\leq \mathcal{E}_2(X)\leq C_2\mathcal{E}_1(X)$ for all $X\in C(G)$. For any $X\in C(G)$, the weighted Laplacian is defined as
\[
\Delta_{\mu} X(i)=\sum_{j\sim i}\frac{\omega_{ij}}{\mu_i}(X(j)-X(i)).
\]
Readers can also view weighted Laplacian as $n\times n$ matrix, two views are the same thing. If $\sum\limits_{j\in\mathcal{N}_i}\omega_{ij}< \mu_{i}$, then we define the aggregation matrix as $P_{\mu}=\Delta_{\mu}+I_n$ with $I_n$ the identity matrix. For $X\in C(G)$, we define the integration of $X$ as 
\[
\int_{G}Xd\mu=\sum\limits_{i\in [n]}X(i)\mu_i.
\]
For any $X,Y\in C(G)$, we define the inner product of gradient as follows.
\[
\nabla_{\mu} X\cdot\nabla_{\mu}Y(i)=\sum\limits_{j\sim i}\frac{\omega_{ij}}{\mu_i}(X(j)-X(i))\cdot(Y(j)-Y(i)).
\]
We abbreviate $\nabla_{\mu} X\cdot\nabla_{\mu}X$ as $\vert \nabla_{\mu}X\vert^2$. For a vector $p=(p_1,...,p_d)$, the $l^2$ norm of vector $p$ is defined as $\vert p\vert=\sqrt{\sum\limits_{i=1}^{d}p_i^{2}}$. For $X\in C(G)$, We define $\Vert X\Vert$ as 
\[
\Vert X\Vert=\sqrt{\sum\limits_{i\in [n]}\vert X(i)\vert^2}
\]
The update rule for attention-based GNNs is given by
\[
X^{k+1}=\sigma(P^{k}X^kW^k).
\]
Where $W^k$ is a learnable weight matrix and $P^k$ is the aggregation matrix at the $k$-th layer. Let $P^k_{ij}$ be the element in the $i$-th row and $j$-th column of $P^k$, which represents the attention value of the edge $(x_i, x_j)$. In attention-based GNNs, $P^k_{ij}$ is defined as:
\[
P_{ij}^k=\frac{exp(e^k_{ij})}{exp(e_{ii}^k)+\sum_{l\in\mathcal{N}_i}exp(e^k_{il})}.
\]
where $e^k_{ij}= g^k(X^{k}_i,X^k_j,\theta^k_{ij})$ with learnable parameters $\theta^k_{ij}$.
We postulate the symmetry $e^k_{ij}=e^k_{ji}$ for theoretical convenience, and define $\omega^k_{ij}=\exp(e^l_{ij})$ along with $\mu_i^k=\exp(e_{ii}^k) + \sum_{l\in\mathcal{N}_i} \exp(e^k_{il})$. This setup allows $P^k$ to be interpreted as a aggregation matrix on a weighted graph $G=(V,E,\omega^k,\mu^k)$. The aggregation matrix in GCN notably satisfies $e^k_{ij}=e^k_{ji}$. Recall the definition of node similarity measure and over-smoothing in \citep{rusch2023survey}.
\begin{definition}(Over-smoothing)
    Let G be an undirected, connected graph, and $X^k\in\mathbb{R}^{n\times d}$ denote the k-th layer hidden features of an N-layer GNN defined on G. Moreover, we call $\mathcal{E}:\mathbb{R}^{n\times d}\rightarrow \mathbb{R}_{\geq 0}$ a node similarity measure if it satisfies the following axioms:
    \begin{itemize}
        \item $\exists c\in\mathbb{R}^{d} $ with $X(i)=c$ for all nodes $i\in V$ if and only if $\mathcal{E}(X)=0$ , for $X\in\mathbb{R}^{n\times d}$
        \item $\mathcal{E}(X+Y)\leq\mathcal{E}(X)+\mathcal{E}(Y)$ , for all $X,Y\in\mathbb{R}^{n\times d}$
    \end{itemize}
    We then define over-smoothing with respect to $\mathcal{E}$ as the layer-wise exponential convergence of the node-similarity measure $\mathcal{E}$ to zero, that is for $k=0,\dots,N$ and some constants $C_1,C_2\textgreater0$
    \begin{equation*}
        \mathcal{E}(X^k)\leq C_1e^{-C_2k}
    \end{equation*}
\end{definition} 
The m-th derivative energy, which is defined as
\[
\mathcal{E}_{m}(X):=\frac{1}{n}\int_{G}\vert\nabla_{\mu}^mX\vert^{2}d\mu,
\]
is known to be a node similarity measure \citep{guan2025measuringoversmoothingdirichletenergy}, where  $m\in \mathbb{N}$. The explicit expression is as follows 
\[
    \vert\nabla_{\mu}^m X\vert^2 = \left\{
    \begin{array}{cc}
           \vert(-\Delta_{\mu})^{\frac{m}{2}}X\vert^2 & \text{if\enspace m\enspace is\enspace even\enspace } \\
           \vert\nabla_{\mu}(-\Delta_{\mu})^{\frac{m-1}{2}}X\vert^2 & \text{if\enspace m\enspace is\enspace odd }\\
           
    \end{array}
    \right.
\]
The case $m=0$ is exactly measure proposed by \cite{NEURIPS2023_6e4cdfdd}, the case $m=1$ is exactly Dirichlet energy and the case $m=2$ will be termed as Laplacian energy. A more familiar expression for Laplacian energy is that
\[
\frac{1}{n}\int_{G}\vert \Delta_{\mu}X\vert^{2}d\mu=\frac{1}{n}\sum_{i\in[n]}\vert\Delta_{\mu}X(i)\vert^2\mu_i.
\]
We will specify $\omega_{ij}=1$ and $\mu_i=deg_i+1$ for Laplacian energy if there is no further specification. Recall that all these measures are equivalent \citep{guan2025measuringoversmoothingdirichletenergy}.
\begin{lemma}\label{compare}
    $G = (V,E,\omega,\mu)$ is a weighted graph. Then for any function X on the graph and all $m_1,m_2\in\mathbb{N}$, There exist universal constants $C_1, C_2\textgreater0$ such that
    \[
        C_2\int_{G}\vert\nabla_{\mu}^{m_1} X\vert^2d\mu\leq\int_{G}\vert\nabla_{\mu}^{m_2}X\vert^2d\mu\leq C_1\int_{G}\vert\nabla_{\mu}^{m_1} X\vert^2d\mu
    \]
\end{lemma}
The \textit{curse of depth}, originally formalized by~\citet{sun2025cursedepthlargelanguage}, refers to the phenomenon where deeper layers contribute negligibly to representation learning. While prior work quantified this effect using accumulated variance, we empirically observe in this work that the Laplacian energy exhibits a similar accumulation pattern, following a power-law growth. To align our analysis with the over-smoothing metric (i.e., Laplacian energy), we characterize the \textit{curse of depth} via the \textit{relative node similarity measure} and formally define it as follows.
\begin{definition}[\textbf{The Curse of Depth}]
    Let $\mathcal{E}$ be the node similarity measure as established in Definition 2.1. Consider an $N$-layer GNN with layer indices $k=0,\dots,N$. We define the \textit{relative node similarity} $\mathcal{R}(X^{k})$ as the relative rate of change in energy:
    \[
    \mathcal{R}(X^{k}) = \frac{\mathcal{E}(X^{k}) - \mathcal{E}(X^{k-1})}{\mathcal{E}(X^{k-1})}.
    \]
    We say that the model suffers from the \textit{curse of depth} with respect to $\mathcal{E}$ if $\mathcal{E}(X^k)$ is a growth function in terms of $k$ and the relative node similarity converge to zero as $k\rightarrow\infty$.
\end{definition}

\section{Pre-LN versus Post-LN}
\subsection{Placement of LN}
We give the special expression for attention of different attention based GNNs in Table 1. We use self-attention neural network(SAN) to denote GNNs using self-attention. 
\begin{table}[H]
  \caption{Expression of attention for different attention based GNNs}
  \label{sample-table}
  \centering
  \begin{tabular}{llllllll}
    \toprule
    Models  & $e_{ij}$  \\
    \midrule
    $\text{GCN}$&$e_{ij}=1$\\
     $\text{GAT}$&$e_{ij}=\text{LeakyReLU}(a^{T}[WX(i)\vert WX(j)])$\\
    SAN&$e_{ij}=\frac{1}{\sqrt{d}}KX(i)\cdot QX(j)$\\
    \bottomrule
  \end{tabular}
\end{table}
\begin{figure}
  \centering
  \includegraphics[width=0.8\textwidth]{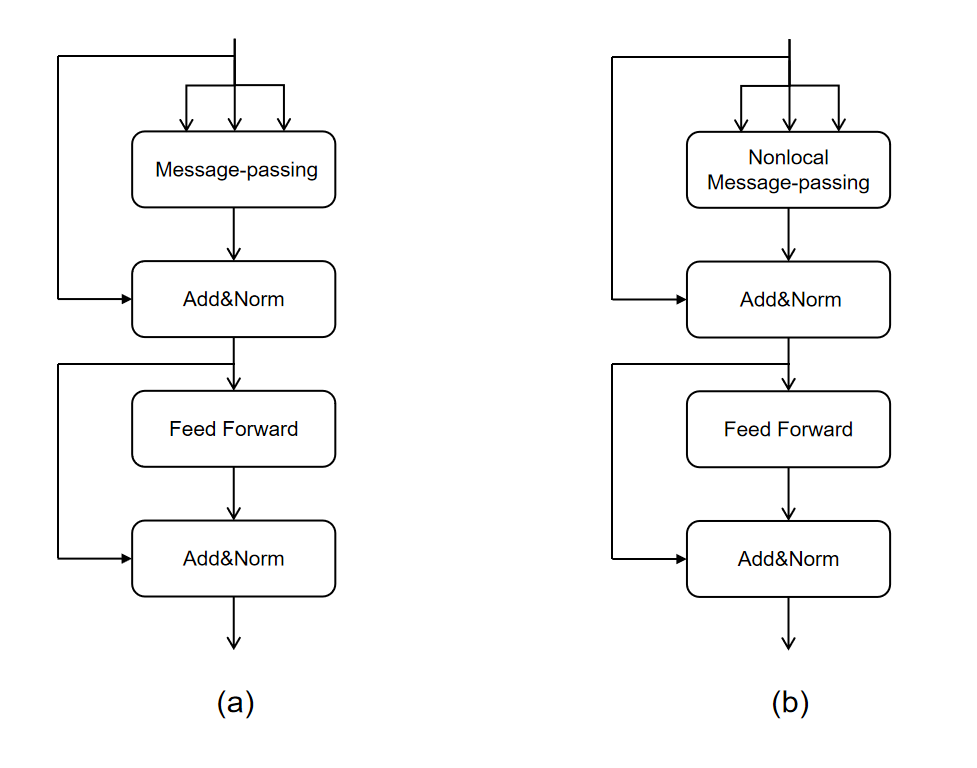}
  \caption{(a) The Post-LN model. (b) Our proposed nonlocal model with Post-LN.}
  \label{model}
\end{figure}
For all models in this work, the encoder is composed of a dropout, and a MLP. The decoder is only a MLP without activation. 
\paragraph{Hidden layers}
A single hidden layer of Post-LN architecture is depicted in the Figure \ref{model}(a). The update formula in message-passing step of Post-LN model is that
\begin{equation*}
    \textbf{MP}(X^{k})=(P^k_1X^{k}V^{k}_1\Vert P^k_2X^{k}V^{k}_2\Vert...\Vert P^k_hX^{k}V^{k}_h)W^k,
\end{equation*}
where h is the number of attention heads, $V^{k}_i$ and $W^{k}$ are learnable matrices. The overall formula of hidden layer is that
\begin{equation*}
    \begin{aligned}
        & Y^{k}= \textbf{Norm}(X^{k}+\textbf{MP}(X^{k}))\\
        & \tilde{Y}^k = FFN(Y^{k})\\
        & X^{k+1}= \textbf{Norm}(Y^{k}+\tilde{Y}^k)
    \end{aligned}
\end{equation*}

Pre-LN is much used in modern LLMs, For the Pre-LN model the position of layer normalization is putted forward. The overall formula for Pre-LN model is as follows
\begin{equation*}
    \begin{aligned}
        & Y^{k}= X^{k}+\textbf{MP}(\textbf{Norm}(X^k))\\
        & X^{k+1}= Y^{k}+\textbf{FFN}(\textbf{Norm}(Y^k))
    \end{aligned}
\end{equation*}
\subsection{Analysis of over-smoothing}
In this subsection, we empirically demonstrate that the Pre-LN SAN architecture is robust to over-smoothing, whereas the Post-LN counterpart is susceptible to it. We focus on the task of fully supervised node classification. To systematically analyze the impact of depth, we evaluate Pre-LN SAN with varying network depths selected from $\{2, 32, 64, 128, 256\}$. Detailed training configurations are provided in the Appendix \ref{Setup}. Our experiments are conducted on five standard benchmark datasets: Cora, Citeseer, and Pubmed \citep{yang2016revisiting}, as well as Coauthor CS and Coauthor Physics \citep{shchur2019pitfallsgraphneuralnetwork}. For all datasets, we employ a random split of $60\%/20\%/20\%$ for training, validation, and testing, respectively. 

To characterize the intrinsic dynamical behavior of the architecture prior to optimization, the following Table 2 presents the Laplacian energy of the final-layer output for the \textbf{untrained} SAN with Pre-LN.
\begin{table}[H]
  \caption{The Laplacian energy of Pre-LN SAN at initialization}
  \label{sample-table}
  \centering
  \begin{tabular}{lllllll}
    \toprule
    Layers  &2 Layers &32 Layers& 64 Layers&128 Layers & 256 Layers \\
    Datasets  \\
    \midrule
    \multirow{1}{*}{Cora} &1.70&6.56&9.86&23.69&23.23 \\
    \multirow{1}{*}{Citeseer} &1.24&5.91&8.24&11.69&14.61   \\
    \multirow{1}{*}{Pubmed}  &0.116&0.194&0.229&0.381&0.404   \\
    \multirow{1}{*}{Physics}  &1.55&6.38&9.72&18.82&31.18 \\
    \multirow{1}{*}{CS} &6.96&39.97&44.41&67.75&104.23  \\
    \bottomrule
  \end{tabular}
\end{table}
As evidenced in the table, the Laplacian energy exhibits a consistent increase as network depth grows. This trend indicates that SAN with Pre-LN effectively prevents the over-smoothing issue typically observed in deep GNNs. In terms of classification performance, we report the training, validation, and test accuracies corresponding to the epoch with the highest validation accuracy (i.e., peak validation performance). Table 3 details the results for the Cora dataset. Due to space constraints, the experimental results for the remaining four datasets are provided in the Appendix \ref{Prenorm}.
\begin{table}[H]
  \caption{Accuracy of Pre-LN SAN on Cora}
  \label{sample-table}
  \centering
  \begin{tabular}{lllllll}
    \toprule
    Layers  &2 Layers &32 Layers& 64 Layers&128 Layers & 256 Layers \\
    Sets &Acc  & Acc &Acc & Acc &Acc \\
    \midrule
    \multirow{1}{*}{Training set} &100&99.08&93.97&97.23&100\\
    \multirow{1}{*}{Validation set}&86.32&86.51&86.69&87.43&87.80 \\
    \multirow{1}{*}{Test set} &88.40&87.11&88.03&86.92&87.29 \\
    \bottomrule
  \end{tabular}
\end{table}
We observe that the model maintains trainability, showing no degradation in training or validation accuracy.

We now turn our attention to the Post-LN SAN configuration. Following the same experimental protocol, we train models with depths $L \in \{2, 32, 64, 128, 256\}$; full training details are provided in the Appendix \ref{Setup}. Table~4 presents the Laplacian energy of the outputs for these \textbf{untrained} models. Additionally, Table~5 reports the classification performance (training, validation, and test accuracy) corresponding to the best validation checkpoint. We report results of Cora here and put others in the Appendix \ref{Prenorm}.
\begin{table}\label{Table:Post-LNLaplacian}
  \caption{The Laplacian energy of Post-LN SAN at initialization}
  \label{sample-table}
  \centering
  \begin{tabular}{lllllll}
    \toprule
    Layers  &2 Layers &32 Layers& 64 Layers&128 Layers & 256 Layers \\
    Datasets  \\
    \midrule
    \multirow{1}{*}{Cora}  &0.558&$1.04\times 10^{-5}$&$2.10\times 10^{-6}$&$3.82\times 10^{-7}$&$1.32\times 10^{-12}$  \\
    \multirow{1}{*}{Citeseer}  &0.301&$1.14\times 10^{-5}$&$6.29\times 10^{-7}$&$6.29\times 10^{-7}$&$1.17\times 10^{-12}$  \\
    \multirow{1}{*}{Pubmed} &$2.41\times 10^{-2}$&$4.08\times 10^{-6}$&$2.57\times 10^{-7}$&$4.62\times 10^{-9}$ &$7.55\times 10^{-12}$    \\
    \multirow{1}{*}{Physics} &0.379&$4.10\times 10^{-5}$&$2.30\times 10^{-7}$ &$1.04\times 10^{-9}$&$7.40\times 10^{-14}$\\
    \multirow{1}{*}{CS} &1.40&$3.71\times 10^{-5}$&$6.60\times 10^{-6}$&$2.57\times 10^{-7}$&$3.10\times 10^{-13}$   \\
    \bottomrule
  \end{tabular}
\end{table}
\begin{table}
  \caption{Accuracy of Post-LN SAN on Cora}
  \label{sample-table}
  \centering
  \begin{tabular}{lllllll}
    \toprule
    Layers  &2 Layers &32 Layers& 64 Layers&128 Layers & 256 Layers \\
    Sets &Acc  & Acc &Acc & Acc &Acc \\
    \midrule
    \multirow{1}{*}{Training set} &97.11&91.32&96.06&84.11&31.65\\
    \multirow{1}{*}{Validation set} &86.14&85.77&85.77&77.82&29.57\\
    \multirow{1}{*}{Test set} &87.66&86.19&86.74&81.58&29.83 \\
    \bottomrule
  \end{tabular}
\end{table}
As evident from the results, the Laplacian energy undergoes an exponential decay with increasing depth, signifying severe over-smoothing of the node features. This feature collapse precipitates a rapid degradation in model performance. Crucially, accuracy drops precipitously not only on the test set but also on the training and validation sets. This indicates that the deep Post-LN architecture suffers from fundamental optimization difficulties, rendering the model effectively untrainable in deeper regimes.

Also, we compare a 256-layer Pre-LN SAN with a 256-layer Post-LN SAN. As illustrated in Figure\ref{fig: log-log plots of Laplacian energy}(a)(c), the Laplacian energy of the layer outputs $X^k$ in the Pre-LN SAN increases. In contrast, the Post-LN SAN exhibits a exponentially decrease in Laplacian energy across layers. This suggests that Pre-LN SANs are biased towards the identity branch, primarily accumulating Laplacian energy while exhibiting diminishing relative changes in deeper layers, which leads to inefficient depth utilization. Conversely, the Post-LN architecture—which applies normalization after the residual addition—effectively mixes information from both the residual and identity branches. However, this aggressive mixing can lead to over-smoothing in deep GNNs.
\section{Understand and Overcome over-smoothing}\label{Migrate}
\subsection{Theoretical analysis}
For Post-LN models, LN plays a role in mixing information of residual branch and identity branch. Over-smoothing is caused by residual branch, hence we can simply consider attention based GNNs. Recall that attention based GNNs can be viewed as the discretization of heat equation \citep{wang2021dissecting}
\[
\frac{\partial X_t}{\partial t}= \Delta_{\mu} X_{t}.
\]
 We recall the reason behind the over-smoothing which is presented in \citep{guan2025measuringoversmoothingdirichletenergy}. Under the heat equation, the Dirichlet energy will evolve as follows.
\[
\frac{\partial \int_{G}\vert\nabla_{\mu} X_t\vert^2d\mu}{\partial t}= -2\int_{G}\vert\Delta_{\mu} X_t\vert^2d\mu .
\]
By Lemma \ref{compare}, for some constant $C_1,C_2$ we have 
\[
C_1\int_{G}\vert\nabla_{\mu} X_t\vert^2d\mu\leq \int_{G}\vert\Delta_{\mu} X_t\vert^2d\mu\leq C_2\int_{G}\vert\nabla_{\mu} X_t\vert^2d\mu
\]
Then 
\[
-2C_2\int_{G}\vert\nabla_{\mu} X_t\vert^2d\mu\leq\frac{\partial \int_{G}\vert\nabla_{\mu} X_t\vert^2}{\partial t}\leq -2C_1\int_{G}\vert\nabla_{\mu} X_t\vert^2d\mu
\]
solving this we have
\[
-2C_2\leq \frac{\partial\big{(}ln(\int_{G}\vert\nabla_{\mu}X_t\vert^2d\mu))\big{)}}{\partial t}\leq -2C_1.
\]
Hence, we have
\[
e^{-2C_2t}\int_{G}\vert\nabla_{\mu} X_0\vert^2d\mu \leq\int_{G}\vert\nabla_{\mu} X_t\vert^2d\mu\leq e^{-2C_1t}\int_{G}\vert\nabla_{\mu} X_0\vert^2d\mu .
\]
By above heuristic computation, if under some evolving equation we have
\[
\frac{\int_{G}\vert\nabla_{\mu}X_t\vert^2d\mu}{\partial t}\sim -(\int_{G}\vert\nabla_{\mu}X_t\vert^2d\mu)^{\beta}
\]
for some $\beta>1$, then we have 
\[
\frac{\partial}{\partial t}(\frac{1}{(\int_{G}\vert\nabla_{\mu}X_t\vert^2d\mu)^{\beta -1}})\sim 1 .
\]
Hence, solving this we will have the algebraically decay rate with some constants $C_1,C_2>0$.
\[
min(C_2\int_{G}\vert\nabla_{\mu}X_0\vert^2d\mu,C_2\frac{1}{t^{\frac{1}{\beta-1}}})\leq \int_{G}\vert\nabla_{\mu}X_t\vert^2d\mu\leq C_1\frac{1}{t^{\frac{1}{\beta-1}}}.
\]
When $\beta\rightarrow1$, the decay rate will be faster than any algebraically decay rate which is exactly exponentially decay rate. In the next theorem, we give a simple evolving equation which can achieve our goal for $\beta=2$.
\begin{theorem}
    Let $G=(V,E,\omega,\mu)$ be a finite graph, $X_0\in\mathbb{R}^{n\times d}$ is the initial feature, then under the non-local diffusion
    \[
    \frac{\partial X_t}{\partial t}=(\int_{G}\vert\Delta_{\mu} X_t\vert^2d\mu)\Delta_{\mu} X_t
    \]
    we have the algebraically decay rate
    \[
min(C_2\int_{G}\vert\nabla_{\mu}X_0\vert^2d\mu,C_2\frac{1}{t})\leq \int_{G}\vert\nabla_{\mu}X_t\vert^2d\mu\leq C_1\frac{1}{t}.
\]
\end{theorem}

We now present our Nonlocal Message-passing Algorithm.
\paragraph{Architecture of nonlocal message passing with Post-LN}
The architecture of hidden layers is shown in Figure\ref{model}(b). Our nonlocal model with Pos-LN only differs in the message-passing step which is called nonlocal message passing. That is 
\begin{equation*}
    \textbf{NonlocalMP}(X^{k})=\bigg{(}\frac{1}{n}\Vert P^{k}_1X^{k}-X^{k}\Vert^2P^k_1X^{k}V^{k}_1\bigg{\Vert} ...\bigg{\Vert} \frac{1}{n}\Vert P^{k}_hX^{k}-X^{k}\Vert^2P^k_hX^{k}V^{k}_h\bigg{)}W^k.
\end{equation*}
That is we multiply Laplacian energy $\frac{1}{n}\Vert P^{k}_iX^{k}-X^{k}\Vert^2$ before each aggregation $P^{k}_iX^{k}$
Intuitively, If over-smoothing occurs and Laplacian energy $\frac{1}{n}\Vert P^{k}_iX^{k}-X^{k}\Vert^2$ is every small, then the residual branch $\frac{1}{n}\Vert P^{k}_iX^{k}-X^{k}\Vert^2P_i^{k}X^{k}V^{k}$ will be small at the same time, hence aggregation term $P^{k}_iX^{k}$ do not influence much.

\paragraph{Computational complexity of our proposed method}
Since $P^{k}_iX^{k}$ is already computed, the extra computation of Laplacian energy only require $O(nd)$.
\subsection{Empirical analysis}
In this subsection, we provide empirical evidence that our proposed method effectively mitigates the over-smoothing issue. Adopting a consistent experimental setting, we report the Laplacian energy at initialization and the classification accuracy corresponding to the best-performing model on the validation set. Due to space limitations, we present the results on the Cora dataset here as a representative example, while detailed results for the remaining datasets are provided in the Appendix \ref{nonlocal}.
\begin{table}[H]
  \caption{The Laplacian energy of nonlocal SAN with Post-LN at initialization}
  \label{sample-table}
  \centering
  \begin{tabular}{lllllll}
    \toprule
    Layers  &2 Layers &32 Layers& 64 Layers&128 Layers & 256 Layers \\
    Datasets \\
    \midrule
    \multirow{1}{*}{Cora} &0.793&$6.47\times 10^{-3}$&$2.36\times 10^{-3}$&$1.28\times 10^{-3}$&$8.55\times 10^{-5}$  \\
    \multirow{1}{*}{Citeseer} &0.489&$8.43\times 10^{-3}$&$2.58\times 10^{-3}$&$2.91\times 10^{-3}$&$4.39\times 10^{-5}$ \\
    \multirow{1}{*}{Pubmed}  &$3.56\times 10^{-2}$&$1.98\times 10^{-3}$&$1.55\times 10^{-3}$&$5.91\times 10^{-5}$&$2.13\times 10^{-5}$  \\
    \multirow{1}{*}{Physics}&0.812&$7.71\times 10^{-3}$&$2.64\times 10^{-3}$&$2.70\times 10^{-3}$&$1.04\times 10^{-4}$\\
    \multirow{1}{*}{CS} &3.26&$3.13\times 10^{-2}$&$1.65\times 10^{-2}$&$4.72\times 10^{-3}$&$3.04\times 10^{-4}$  \\
    \bottomrule
  \end{tabular}
\end{table}
\begin{table}[H]
  \caption{Accuracy of nonlocal SAN with Post-LN on Cora}
  \label{sample-table}
  \centering
  \begin{tabular}{lllllll}
    \toprule
    Layers  &2 Layers &32 Layers& 64 Layers&128 Layers & 256 Layers \\
    Sets &Acc  & Acc &Acc & Acc &Acc \\
    \midrule
    \multirow{1}{*}{Training set}&100&92.98&91.69&90.89&92.12 \\
    \multirow{1}{*}{Validation set} &85.77&87.62&86.14&87.43&85.40\\
    \multirow{1}{*}{Test set}  &88.77&87.29&87.11&88.03&87.85 \\
    \bottomrule
  \end{tabular}
\end{table}
We observe that the Laplacian energy decays gradually. Crucially, unlike Post-LN SAN, our proposed nonlocal SAN with Post-LN facilitates trainability in deeper layers.
\section{Analysis of the curse of depth}
 While modern LLMs predominantly adopt Pre-LN to mitigate issues such as gradient instability, recent studies suggest that Pre-LN Transformers may suffer from inefficient depth utilization. \citep{csordás2025languagemodelsusedepth} analyzed models including Llama 3.1, Qwen 3, and QLMo 2, measuring the contribution norm of each layer relative to the residual stream. They observed that beyond the middle layers, the models tend to amplify existing features rather than substantially modifying the residual stream, resulting in diminishing representation updates in deeper layers. Similar phenomena in the Qwen 2.5 family have been reported by \citep{hu2025what}.

\citep{sun2025cursedepthlargelanguage} formalized this limitation as the 'curse of depth' in LLMs, which refers to the marginal contribution of deeper layers to both learning and representation. By performing layer pruning without fine-tuning, they demonstrated that removing deeper layers in Pre-LN models (e.g., Qwen 3) incurs only minor performance drops—or even performance gains—whereas Post-LN models (e.g., BERT) suffer significant degradation. Furthermore, analysis of the angular distance between layer representations revealed that deeper layers in Pre-LN models produce outputs highly similar to their preceding layers, indicating a high degree of redundancy.

In this section, we provide empirical evidence within the context of GNNs, corroborating the findings of \citep{sun2025cursedepthlargelanguage}. Figure \ref{fig: inter-layer cosine similarity} presents the cosine similarity between layer representations. For two layer representations $X^{s}, X^{t} \in \mathbb{R}^{n \times d}$, the cosine similarity is defined as:
\begin{equation}
    \text{sim}(X^{s},X^{t}) = \frac{1}{n}\sum_{i\in[n]}\frac{ X^{s}(i)\cdot X^{t}(i) }{| X^{s}(i) | | X^{t}(i) |}.
\end{equation}
 In the non-local SAN with Post-LN, representations from distant layers display low similarity, indicating continuous feature transformation. In contrast, deeper layers in the Pre-LN SAN exhibit increasingly high similarity, reinforcing the hypothesis of reduced effective feature transformation in deep Pre-LN models.

Also, we perform a layer sensitivity analysis to assess the impact of specific layers on the final output. We prune individual layers from the set $\{2, 32, 64, 96, 128, 160, 192, 244\}$ and evaluate the resulting test accuracy, as detailed in Table~8. For the nonlocal SAN with Post-LN, pruning any intermediate layer leads to a degradation in accuracy compared to the baseline of 87.85\%, implying that all layers contribute to the learned representation. Strikingly, for the Pre-LN model, pruning layers beyond depth 64 has a negligible impact on accuracy. This finding indicates that deep layers in the Pre-LN architecture contribute minimally to the final representation, effectively validating the existence of the 'curse of depth'
\begin{table}[H]
  \caption{Impact of pruning individual layers on the accuracy(Cora)}
  \label{sample-table}
  \centering
  \begin{tabular}{lllllllll}
    \toprule
    Layers  &2  &32& 64 &96&128  &160&192 &224   \\
    methods &Acc  & Acc &Acc & Acc &Acc &Acc&Acc&Acc\\
    \midrule
    \multirow{1}{*}{nonlocal SAN with Post-LN} &87.66&87.11&86.74&86.92&87.29&87.11&86.92&87.66 \\
    \multirow{1}{*}{Pre-LN SAN} &81.58&86.56&87.29&87.29&87.29&87.29&87.29&87.29\\
    \bottomrule
  \end{tabular}
\end{table}
Complementing the aforementioned layer sensitivity analysis, we further characterize the \textit{curse of depth} via the power-law growth of the realative Laplacian energy. we already give a definition for the curse of depth in Definition 2.1. As illustrated in Figure \ref{fig: log-log plots of Laplacian energy}(c), the evolution of Laplacian energy exhibits a near-linear trend on a log-log scale, indicative of power-law growth. Corroborating evidence for additional datasets is provided in the Appendix \ref{curse}. Hence these results validate our definition. Theoretically, we can derive a estimate of decat rate of relative Laplacian energy for the continuous formulation of Pre-LN models in Theorem 5.1. 
\begin{theorem}
    For the evolving equation 
    \[
\frac{\partial X_t}{\partial t}=P_{\mu}(Norm(X_t))
\]
There exist constant $C_1>0$ such that we have $\int_{G}\vert \nabla X_t\vert^2d\mu \leq(C_1t +\sqrt{\int_{G}\vert \nabla_{\mu}X_0\vert^2d\mu})^2$. Furthermore, assuming a power-law lower bound for $\int_{G}\vert \nabla_{\mu} X_t\vert^2d\mu$
\[
C_2 t^{q}\leq \int_{G}\vert \nabla_{\mu} X_t\vert^2d\mu, 
\]
we can get a power-law decay for relative energy
\[
\frac{1}{\int_{G}\vert \nabla_{\mu} X_t\vert^2d\mu }\frac{\partial \int_{G}\vert \nabla_{\mu} X_t\vert^2d\mu }{\partial t}\leq C\frac{1}{t^{\frac{q}{2}}}.
\]
\end{theorem}
\section{Conclusion}
We develop a simple method called nonlocal message passing to theoretically and empirically overcome over-smoothing efficiently. This method can be also easily applied to any problem with respect to exponentially converge to steady state. Also, the distinct phenomenon shed light to addressing the curse of depth in LLMs and further study on Post-LN configuration in LLMs. Most importantly, more works need to be done in the future to train a large and universal graph neural networks.
\medskip

\bibliographystyle{apalike}

\small


\appendix

\section{Proofs of Theorems}
\subsection{Preliminaries}
We first give an integration by parts lemma.
\begin{lemma}
    Let $G=(V,E,\omega,\mu)$ be a weighted graph, then we have integration by parts
    \[
    \int_{G}-\Delta_{\mu}X\cdot Yd\mu =\int_{G}\nabla_{\mu}X\cdot\nabla_{\mu}Yd\mu
    \]
\end{lemma}
\begin{proof}
    Given a weighted graph $G=(V,E,\omega,\mu)$ and vectored-valued functions $X,Y: V\rightarrow\mathbb{R}^{d}$. First, we suppose $d=1$, we have the following equality
    \begin{align*}
  &\sum_{i\in[n]}\sum_{j\in\mathcal{N}_i}\omega_{ij}X(i)Y(i)\\
        &=\sum_{i\sim j}\omega_{ij}(X(i)Y(i)+X(j)Y(j))\\
        &=\sum_{i\in[n]}\sum_{j\in\mathcal{N}_i}\omega_{ij}X(j)Y(j)
    \end{align*}
    and 
    \begin{align*}
    &\sum_{i\in[n]}\sum_{j\in\mathcal{N}_i}\omega_{ij}X(j)Y(i)\\
        &=\sum_{i\sim j}\omega_{ij}(X(j)Y(i)+X(i)Y(j))\\
        &=\sum_{i\in[n]}\sum_{j\in\mathcal{N}_i}\omega_{ij}X(i)Y(j)
    \end{align*}
    Combining above equalities we deduce
    \begin{align*}
      \int_{V}\Delta_{\mu} X\cdot Yd\mu 
      &= \sum_{i=1}^n\sum_{j\in\mathcal{N}_i}\omega_{ij}(X(j)-X(i))Y(i)\\
         & = \sum_{i=1}^n\sum_{j\in\mathcal{N}_i}(\omega_{ij}X(j)Y(i)-\omega_{ij}X(i)Y(i))\\
         & = \sum_{i=1}^n\sum_{j\in\mathcal{N}_i}(\omega_{ij}X(i)Y(j)-\omega_{ij}X(j)Y(j))\\
         &=\sum_{i=1}^n\sum_{j\in\mathcal{N}_i}\omega_{ij}(X(i)-X(j))Y(j)\\       
    \end{align*}

    Thus adding the right-hand side of the first line and last line we get
    \begin{equation*}
        \begin{aligned}
        \int_{G}\Delta_{\mu} X\cdot Yd\mu 
        &= \sum_{i=1}^n\sum_{j\in\mathcal{N}_i}\frac{\omega_{ij}}{2}(X(j)-X(i))(Y(i)-Y(j))\\&=-\int_{V}\nabla_{\mu} X\cdot\nabla_{\mu} Y
    \end{aligned}
    \end{equation*}

    Now suppose general $d$ and $X=(X_1,\dots,X_d)^T, Y=(Y_1,\dots,Y_d)^T$, then a direct calculation yield 
    \begin{align*}
        &\int_G\Delta_{\mu} X\cdot Yd\mu\\
        &=\int_G\Delta_{\mu} X_1\cdot Y_1d\mu+\dots+\Delta_{\mu} X_d\cdot Y_dd\mu\\
        &=-\int_G\nabla_{\mu} X_1\cdot\nabla_{\mu} Y_1d\mu-\dots-\nabla_{\mu} X_d\cdot \nabla_{\mu}Y_dd\mu\\
        &=-\int_G\nabla_{\mu} X\cdot\nabla_{\mu} Yd\mu
    \end{align*}
\end{proof}
\subsection{Supplemental proof of subsection 4.1}
We only need to prove \[
\frac{\partial \int_{G}\vert\nabla_{\mu} X_t\vert^2d\mu}{\partial t}= -\int_{G}\vert\Delta_{\mu} X_t\vert^2d\mu .
\]
This is from Lemma A.1 and a direct computation.
\begin{align*}
    \frac{\partial \int_{G}\vert\nabla_{\mu} X_t\vert^2d\mu}{\partial t} &= -2\int_{G}\Delta_{\mu}X_t\cdot\frac{\partial X_t}{\partial t}d\mu\\
    &=-2\int_{G}\vert\Delta_{\mu}X_t\vert^2d\mu
\end{align*}

\subsection{Proof of Theorem 4.1}
We only need to prove that for some constants $C_1,C_2>0$ we have
\[
-C_1(\int_{G}\vert\nabla_{\mu}X_t\vert^2d\mu)^{2}\leq\frac{\int_{G}\vert\nabla_{\mu}X_t\vert^2d\mu}{\partial t}\sim -C_2(\int_{G}\vert\nabla_{\mu}X_t\vert^2d\mu)^{2}
\]
Similarly by Lemma A.1 we have
\begin{align*}
    \frac{\partial \int_{G}\vert\nabla_{\mu} X_t\vert^2d\mu}{\partial t} &= -2\int_{G}\Delta_{\mu}X_t\cdot\frac{\partial X_t}{\partial t}d\mu\\
    &=-2(\int_{G}\vert\Delta_{\mu}X_t\vert^2d\mu)^2.
\end{align*}
Since there exists constants $C_1.C_2>0$ such that
\[
C_1 \int_{G}\vert\nabla_{\mu}X_t\vert^2d\mu\leq\int_{G}\vert\Delta_{\mu}X_t\vert^2d\mu\leq C_2 \int_{G}\vert\nabla_{\mu}X_t\vert^2d\mu
\]
We have the desired results.
\subsection{Proof of Theorem 5.1}
Suppose the $l^2$ eigenvalues of operator $-\Delta_{\mu}:\mathbb{R}^{n\times d}\rightarrow\mathbb{R}^{n\times d}$ 
are 
\[
\lambda_1,...\lambda_{nd}
\]
and the corresponding eigenfunctions
\[
f_{1},...,f_{nd}
\]
for fixed $t$ suppose the spectral decomposition of $X_t$ is 
\[
X_t =\sum\limits_{i}C_if_i
\]
the spectral decomposition of $Norm(X_t)$ is
\[
Norm(X_t)=\sum\limits_{i}B_if_i
\]
then we have
\begin{align*}
\frac{\partial\int_{G}\vert \nabla_{\mu}X_t\vert^2d\mu}{\partial t}=&\int_{G}-2P_{\mu}\Delta_{\mu}X_t\cdot proj_{S^d}X_td\mu\\
&=\sum\limits_{i}2(1-\lambda_i)\lambda_iB_iC_i\\
&\leq C\sum\limits_{i}\sqrt{\lambda_i}\vert B_iC_i\vert\\
&\leq C\sqrt{(\sum\limits_{i}B_i^2)(\sum\limits_{i}\lambda_iC_i^2)}\\
\end{align*}
Since 
\[\sum\limits_{i}\lambda_iC_i^2=\int_{G}\vert \nabla_{\mu}X_t\vert^2d\mu\]
and 
\[
\sum\limits_{i}B_i^2=\int_{G}\vert Norm(X_t)\vert^2d\mu=n
\]
Hence
\[
\frac{\partial\int_{G}\vert \nabla_{\mu}X_t\vert^2d\mu}{\partial t}\leq C\sqrt{\int_{G}\vert \nabla_{\mu}X_t\vert^2d\mu}
\]
solving this we have 
\[
\int_{G}\vert \nabla_{\mu}X_t\vert^2d\mu\leq (C_1t +\sqrt{\int_{G}\vert \nabla_{\mu}X_0\vert^2d\mu})^2
\]
Now assuming a lower power-law bound
\[
C_2 t^{q}\leq \int_{G}\vert \nabla_{\mu} X_t\vert^2d\mu, 
\]
For the relative Laplacian energy, we directly compute that
\begin{align*}
    \frac{1}{\int_{G}\vert \nabla_{\mu} X_t\vert^2d\mu }\frac{\partial \int_{G}\vert \nabla_{\mu} X_t\vert^2d\mu }{\partial t}\leq \frac{C}{\sqrt{\int_{G}\vert \nabla_{\mu} X_t\vert^2d\mu} }\leq \frac{C}{t^{\frac{q}{2}}}.
\end{align*}
\section{Experiment details}

\subsection{Experiment Setup}\label{Setup}
For all models, we train 5000 epoches and the L2 regularization is 5e-4. Hidden dimension is fixed as 32. For Pre-LN, the learning rate is fixed as 1e-3. For Post-LN, we use the warm up strategy for learning rate with maximum 1e-4. 
\subsection{Data details}
\begin{table}[H]
  \caption{Dataset Statistics}
  \label{sample-table}
  \centering
  \begin{tabular}{llllll}
    \toprule
    Datasets &Cora &Citeseer &Pubmed&CS&Physics\\
    \midrule
    Nodes &2708&3327&19717&18333&34493\\
    Edges &5278&4522&44324&81894&247962\\
    Features &1433&3703&500&6805&8415\\
    Classes  &7&6&3&15&5\\
    \bottomrule
  \end{tabular}
\end{table}
\subsection{Supplemental results in subsection 3.2}\label{Prenorm}
The training, validation and test accuracy of SAN with Pre-LN are as follow.
\begin{table}[H]
  \caption{Train accuracy of SAN with Pre-LN over different datasets }
  \label{sample-table}
  \centering
  \begin{tabular}{lllllll}
    \toprule
    Layers  &2 Layers &32 Layers& 64 Layers&128 Layers & 256 Layers \\
    Datasets &Acc  & Acc &Acc & Acc &Acc \\
    \midrule
    \multirow{1}{*}{Cora} &100&99.08&93.97&97.23&100\\
    \multirow{1}{*}{Citeseer} &100&89.98&100&99.90&89.23  \\
    \multirow{1}{*}{Pubmed} &94.51&97.02&99.64&97.44&97.43    \\
    \multirow{1}{*}{Physics} &100&100&98.51&99.64&99.45 \\
    \multirow{1}{*}{CS} &100&97.44&99.99&98.17&97.78  \\
    \bottomrule
  \end{tabular}
\end{table}
\begin{table}[H]
  \caption{Validation accuracy of Transformer with Pre-LN}
  \label{sample-table}
  \centering
  \begin{tabular}{lllllll}
    \toprule
    Layers  &2 Layers &32 Layers& 64 Layers&128 Layers & 256 Layers \\
    Datasets &Acc  & Acc &Acc & Acc &Acc \\
    \midrule
    \multirow{1}{*}{Cora} &86.32&86.51&86.69&87.43&87.80 \\
    \multirow{1}{*}{Citeseer} &74.44&75.64&75.04&73.53&73.98 \\
    \multirow{1}{*}{Pubmed}  &90.03&89.35&89.45&89.78&89.78    \\
    \multirow{1}{*}{Physics} &97.14&96.58&96.71&96.36&96.78\\
    \multirow{1}{*}{CS} &94.38&92.55&92.06&92.03&92.39  \\
    \bottomrule
  \end{tabular}
\end{table}
\begin{table}[H]
  \caption{Test accuracy of SAN with Pre-LN}
  \label{sample-table}
  \centering
  \begin{tabular}{lllllll}
    \toprule
    Layers  &2 Layers &32 Layers& 64 Layers&128 Layers & 256 Layers \\
    Datasets &Acc  & Acc &Acc & Acc &Acc \\
    \midrule
    \multirow{1}{*}{Cora} &88.40&87.11&88.03&86.92&87.29  \\
    \multirow{1}{*}{Citeseer} &72.07&72.37&70.87&71.17&72.67  \\
    \multirow{1}{*}{Pubmed}  &89.00&89.40&88.39&89.17&89.43   \\
    \multirow{1}{*}{Physics} &96.64&96.09&96.16&95.70&96.26\\
    \multirow{1}{*}{CS}  &94.77&93.29&92.64&92.53&93.21 \\
    \bottomrule
  \end{tabular}
\end{table}
The training, validation and test accuracy of SAN with Post-LN are as follow.
\begin{table}[H]
  \caption{Train accuracy of SAN with Post-LN over different datasets}
  \label{sample-table}
  \centering
  \begin{tabular}{lllllll}
    \toprule
    Layers  &2 Layers &32 Layers& 64 Layers&128 Layers & 256 Layers \\
    Datasets &Acc & Acc&Acc & Acc &Acc \\
    \midrule
     \multirow{1}{*}{Cora}  &97.11&91.32&96.06&84.11&31.65\\
    \multirow{1}{*}{Citeseer} &94.49&88.03&90.48&86.02&55.96     \\
    \multirow{1}{*}{Pubmed}   &97.13&85.75&85.16&70.14&40.18    \\
    \multirow{1}{*}{Physics}  &100&97.36&95.23&76.12&50.84 \\
    \multirow{1}{*}{CS}  &100&94.53&92.36&75.66&23.14  \\
    \bottomrule
  \end{tabular}
\end{table}

\begin{table}[H]
  \caption{Validation accuracy of Transformer with Post-LN over different datasets}
  \label{sample-table}
  \centering
  \begin{tabular}{lllllll}
    \toprule
    Layers  &2 Layers &32 Layers& 64 Layers&128 Layers & 256 Layers \\
    Datasets &Acc  & Acc &Acc & Acc &Acc \\
    \midrule
     \multirow{1}{*}{Cora}  &86.14&85.77&85.77&77.82&29.57\\
    \multirow{1}{*}{Citeseer}  &74.74&72.93&75.49&73.98&54.59    \\
    \multirow{1}{*}{Pubmed}   &89.88&83.87&84.02&70.94&39.23    \\
    \multirow{1}{*}{Physics}  &96.74&95.42&94.61&76.37&50.13 \\
    \multirow{1}{*}{CS}  &93.51&89.77&89.33&76.32&24.99  \\
    \bottomrule
  \end{tabular}
\end{table}

\begin{table}[H]
  \caption{Accuracy of SAN with Post-LN }
  \label{sample-table}
  \centering
  \begin{tabular}{lllllll}
    \toprule
    Layers  &2 Layers &32 Layers& 64 Layers&128 Layers & 256 Layers \\
    Datasets &Acc  & Acc &Acc & Acc &Acc \\
    \midrule
     \multirow{1}{*}{Cora} &87.66&86.19&86.74&81.58&29.83 \\
    \multirow{1}{*}{Citeseer} &73.12&71.62&71.77&70.42&51.65     \\
    \multirow{1}{*}{Pubmed}   &89.55&83.62&83.01&71.15&39.93    \\
    \multirow{1}{*}{Physics} &96.39&94.58&94.19&75.77&49.96  \\
    \multirow{1}{*}{CS}  &94.06&90.79&89.59&76.23&23.39  \\
    \bottomrule
  \end{tabular}
\end{table}
\subsection{Supplemental results in subsection 4.2}\label{nonlocal}
The training, validation and test accuracy of nonlocal SAN with Post-LN are as follow.
\begin{table}[H]
  \caption{Train accuracy of nonlocal SAN with Post-LN}
  \label{sample-table}
  \centering
  \begin{tabular}{lllllll}
    \toprule
    Layers  &2 Layers &32 Layers& 64 Layers&128 Layers & 256 Layers \\
    Datasets &Acc  & Acc &Acc & Acc &Acc \\
    \midrule
    \multirow{1}{*}{Cora} &100&92.98&91.69&90.89&89.35 \\
    \multirow{1}{*}{Citeseer} &99.25&87.93&87.68&89.38&89.43  \\
    \multirow{1}{*}{Pubmed} &95.84&96.98&93.92&95.00&95.13     \\
    \multirow{1}{*}{Physics} &100&97.70&95.95&94.07&98.10 \\
    \multirow{1}{*}{CS} &100&95.60&94.91&95.53&94.67  \\
    \bottomrule
  \end{tabular}
\end{table}

\begin{table}[H]
  \caption{validation accuracy of nonlocal SAN with Post-LN}
  \label{sample-table}
  \centering
  \begin{tabular}{lllllll}
    \toprule
    Layers  &2 Layers &32 Layers& 64 Layers&128 Layers & 256 Layers \\
    Datasets &Acc  & Acc &Acc & Acc &Acc \\
    \midrule
    \multirow{1}{*}{Cora} &85.77&87.62&86.14&87.43&87.25 \\
    \multirow{1}{*}{Citeseer} &73.23&72.78&73.68&72.78&73.53  \\
    \multirow{1}{*}{Pubmed}  &90.29&89.04&87.67&87.98&87.60    \\
    \multirow{1}{*}{Physics} &96.88&95.49&95.22&93.53&95.42  \\
    \multirow{1}{*}{CS} &93.92&90.48&90.78&90.75&89.91  \\
    \bottomrule
  \end{tabular}
\end{table}
\begin{table}[H]
  \caption{Test accuracy of nonlocal SAN with Post-LN}
  \label{sample-table}
  \centering
  \begin{tabular}{lllllll}
    \toprule
    Layers  &2 Layers &32 Layers& 64 Layers&128 Layers & 256 Layers \\
    Datasets &Acc  & Acc &Acc & Acc &Acc \\
    \midrule
    \multirow{1}{*}{Cora} &88.77&87.29&87.11&88.03&87.85 \\
    \multirow{1}{*}{Citeseer} &69.22&71.77&72.52&71.32&71.62   \\
    \multirow{1}{*}{Pubmed} &89.58&88.79&87.45&88.24&87.53  \\
    \multirow{1}{*}{Physics} &96.45&94.84&94.39&93.32&94.99  \\
    \multirow{1}{*}{CS} &94.27&91.09&91.22&91.47&90.68  \\
    \bottomrule
  \end{tabular}
\end{table}
\section{The curse of depth}\label{curse}
For Pre-LN SAN, we give the evolution of Laplacian energy on Citeseer, Pubmed, CS, Physics at initialization as follow
\begin{figure}[H]
  \centering
  \includegraphics[width=0.8\textwidth]{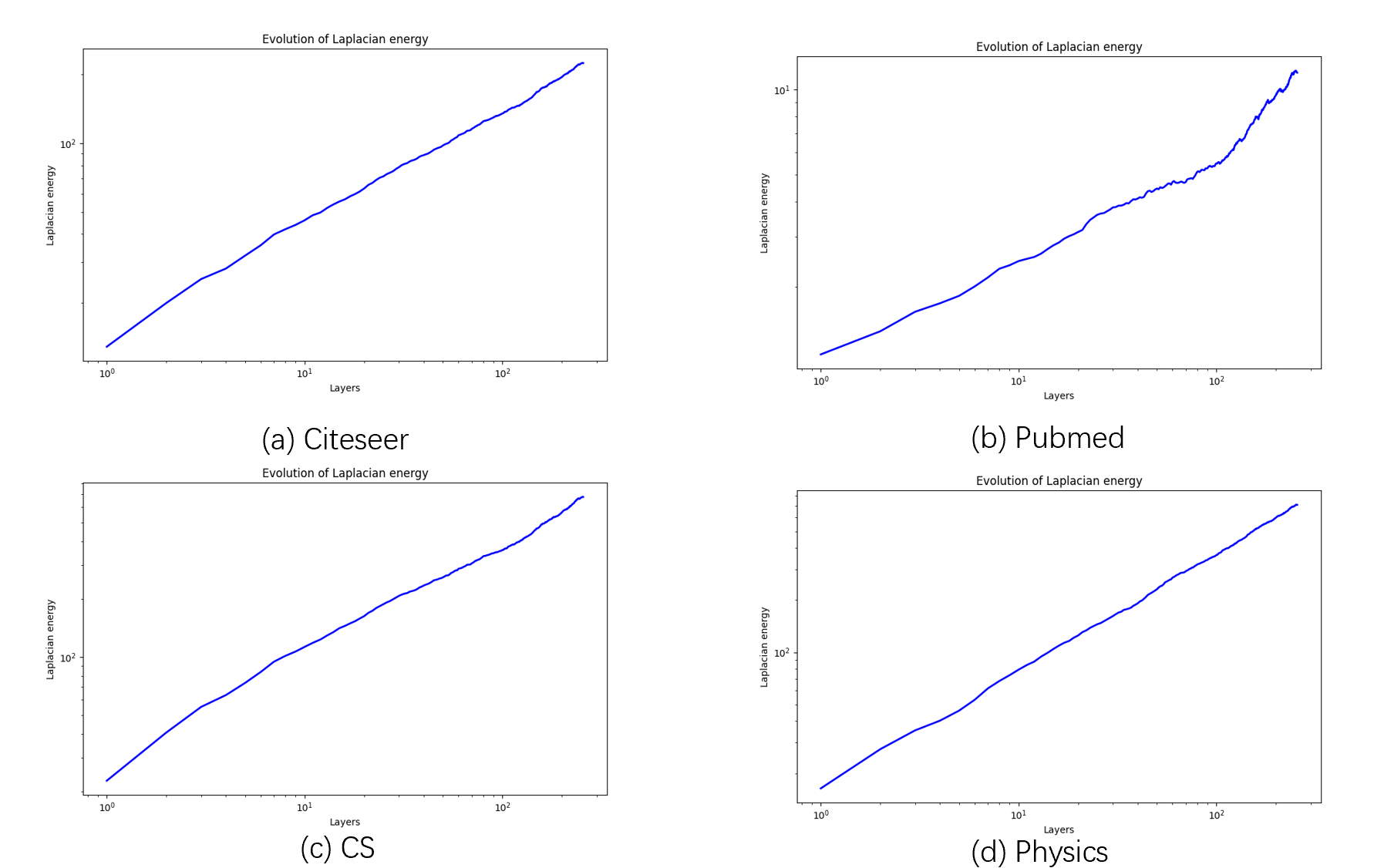}
  \caption{Evolution of Laplacian energy at initialization.}
\end{figure}
\end{document}